\documentclass[11pt,oneside]{article}

\usepackage[T1]{fontenc} %
\usepackage{textcomp} %
\usepackage{lmodern} %
\DeclareTextCommand{\nobreakspace}{T1}{\leavevmode\nobreak\ } %
\usepackage{amsmath} %
\usepackage{amssymb,amsfonts} %
\usepackage{bm} %

\usepackage[paper=a4paper, hscale=0.75, vscale=0.79]{geometry} %

\usepackage{graphicx} %
\graphicspath{{./img/}} %
\DeclareGraphicsExtensions{.pdf,.png,.jpg,.eps} %

\usepackage[dvipsnames]{xcolor} 

\usepackage{algorithm}

\usepackage{booktabs} %
\usepackage{multirow} %

\usepackage{natbib}
\usepackage{amsmath} 
\usepackage{amssymb} 
\usepackage{amsfonts} 
\usepackage{bm} 
\usepackage{amsthm} 
\theoremstyle{definition} \newtheorem{defn}{Definition}
\theoremstyle{plain} \newtheorem{prop}[defn]{Proposition}
\theoremstyle{plain} 
\theoremstyle{plain} 
\theoremstyle{plain} 
\theoremstyle{remark} 
\theoremstyle{remark}

\newcommand{\term}[1]{\textcolor{BlueViolet}{\textit{{#1}}}}

\newcommand*{\defeq}{\mathrel{\vcenter{\baselineskip0.5ex \lineskiplimit0pt     
                     \hbox{\scriptsize.}\hbox{\scriptsize.}}}=}

\newcommand{\overbar}[1]{\mkern 1.5mu\overline{\mkern-1.5mu#1\mkern-1.5mu}\mkern 1.5mu}

\DeclareMathOperator*{\argmin}{arg\,min}
\DeclareMathOperator*{\argmax}{arg\,max}

\newcommand{\Abs}[1]{\lVert{#1}\rVert} 
\newcommand{\abs}[1]{\lvert{#1}\rvert} 
\DeclareMathOperator{\CBE}{CBE} 
\DeclareMathOperator{\COCE}{COCE} 
\def\colonset{:\,} 
\def\cond{\,\vert\,} 
\def\ddist{\mu} 
\def\exx{\mathbf{E}} 
\DeclareMathOperator{\GBL}{GBL} 
\def\HH{\mathcal{H}} 
\def\II{\mathcal{I}} 
\newcommand{\prr}[1]{\mathbf{P}{#1}} 
\newcommand{\rdv}[1]{\mathsf{#1}} 
\def\RR{\mathbb{R}} 
\DeclareMathOperator{\sign}{sign} 
\def\Scal{\mathcal{S}} 
\DeclareMathOperator{\SDRO}{SDRO} 
\def\vaa{\mathbf{V}} 
\def\XX{\mathcal{X}} 
\def\YY{\mathcal{Y}} 
\def\ZZ{\mathcal{Z}} 

\usepackage[pdfauthor={Matthew J. Holland},%
colorlinks=true,%
linkcolor=blue,%
citecolor=blue]{hyperref}
\hypersetup{pdftitle={Making Robust Generalizers Less Rigid with Loss Concentration}} 

\begin{document}

\title{\textbf{Making Robust Generalizers Less Rigid\\with Loss Concentration}}

\author{
  Matthew J.~Holland\\
  Kansai University
  \and
  Toma Hamada\\
  Osaka University
}
\date{} 

\maketitle

\begin{abstract}
While the traditional formulation of machine learning tasks is in terms of performance on average, in practice we are often interested in how well a trained model performs on rare or difficult data points at test time. To achieve more robust and balanced generalization, methods applying sharpness-aware minimization to a subset of worst-case examples have proven successful for image classification tasks, but only using overparameterized neural networks under which the relative difference between ``easy'' and ``hard'' data points becomes negligible. In this work, we show how such a strategy can dramatically break down under simpler models where the difficulty gap becomes more extreme. As a more flexible alternative, instead of typical sharpness, we propose and evaluate a training criterion which penalizes poor loss concentration, which can be easily combined with loss transformations such exponential tilting, conditional value-at-risk (CVaR), or distributionally robust optimization (DRO) that control tail emphasis.
\end{abstract}

\tableofcontents

\section{Introduction}\label{sec:intro}

Towards the ultimate goal of designing reliable, transparent, and trustworthy machine learning systems, the term ``robustness'' is used frequently to describe a wide spectrum of properties that one might expect a learning algorithm or trained model to satisfy. At the core of most of the literature related to robust learning is the following question:
\begin{align*}\label{eqn:RQ}
\text{\textit{``How should rare or difficult data points be treated during training?''}}\tag{RQ}
\end{align*}
While ``difficult'' is a broad term, we use it to describe data on which a given model tends to incur relatively large losses, which cannot be easily improved without sacrificing performance on other data. The traditional answer to the question (\ref{eqn:RQ}) is that the influence of such points on the learning procedure should be reduced, if not removed entirely. This viewpoint treats such points as anomalous outliers that are not relevant to the task at hand. It has roots in classical robust statistics, but is apparent in an active line of work over the past decade on machine learning under heavy-tailed distributions.\footnote{A classic textbook reference is \citet{huber2009a}. See papers by \citet{hsu2014a,holland2017a,holland2019c,nazin2019a,cutkosky2021aNeurIPS} for some representative examples.}

The alternative answer to (\ref{eqn:RQ}) is that rare or difficult data points should be \emph{emphasized}, rather than suppressed, often in pursuit of robustness to adversarial examples, decisions that are more ``fair'' with respect to under-represented groups, or stronger guarantees under worst-case scenarios.\footnote{See for example \citet{hashimoto2018a}, \citet{williamson2019a}, and \citet{duchi2021a}.} One inherent challenge that arises with this latter viewpoint is the problem of how much emphasis to place on the difficult points. When these points are rare, emphasizing them usually means suppressing or ignoring the majority of the training data, which has a serious impact on generalization to test data. This key motivation has led to a fruitful line of work in the past few years aimed at achieving a better balance between performance on average-case and worst-case data.\footnote{See \citet{zhai2021a} and \citet{robey2022a}, plus the references therein.}

Refining our focus to the task of image classification, one particularly salient example from this line of research is the ``SharpDRO'' algorithm proposed by \citet{huang2023a}. In principle, SharpDRO seeks to minimize sharpness in the sense of \citet{foret2021a}, but instead of computing sharpness for the usual expected loss, SharpDRO conditions on the degree of image data corruption, and targets the worst-case conditional expected loss. In practice, the authors diverge somewhat from this objective, effectively minimizing a sum of the expected loss and a ``softened'' sharpness term reflecting all corruption levels. This practical variant of SharpDRO shows state-of-the-art test accuracy over images spanning a wide range of corruption types, albeit with the obvious drawback that corruption levels must be known in advance. In addition, all evaluations of SharpDRO in the existing literature are within the very narrow scenario where the classifier is an overparameterized neural network that can perfectly classify all points, regardless of the degree of data corruption. This means that with sufficient training, the gap between ``easy'' and ``hard'' data points becomes essentially negligible. On the other hand, under simpler, less expressive models, one expects the risk of over-fitting to grow, since the performance on difficult points gets relatively worse, and the loss distribution becomes more heavy-tailed, all else equal. How do so-called ``robust methods'' hold up when the distribution of relative difficulty changes?

With the aforementioned limitations as context, in this work we show how the \emph{balanced accuracy} of SharpDRO degrades sharply under less expressive models. As an alternative method, we propose combining a worst-case emphasizing mechanism (transforming individual losses) with a loss aggregator that explicitly penalizes poor concentration of the transformed losses. Making the loss transformation lets us target difficult points without prior knowledge of data corruption levels, and the aggregation technique works as a low-cost alternative to directly minimizing sharpness. In particular, from our proposed objective function, we can derive a soft ascent-descent update under gradient-based algorithms, with explicit links to sharpness in terms of both learnable parameters \emph{and} data.

We formulate our problem of interest, discuss related literature, give a detailed introduction to the key benchmark method (SharpDRO) in \S{\ref{sec:background}}, and discuss its limitations. We then introduce our method called ``concentrated OCE'' (COCE) and show explicit theoretical links to sharpness-aware learning in \S{\ref{sec:coce}}. In \S{\ref{sec:empirical}} we describe our empirical test design and share key findings. Our first main finding is that as hypothesized, under a simple CNN model with far fewer parameters than the models tested by \citet{huang2023a}, the balanced accuracy obtained by SharpDRO shows a drastic drop, whereas our proposed COCE method significantly outperforms all other benchmarks under the same setup. An additional finding is that for smaller models, our proposed method tends to outperform even SAM \citep{foret2021a} in terms of (non-balanced) test accuracy under settings with and without distribution shift, offering as-good or better performance with half the gradient computation cost. One limitation to our method is that the ideal per-loss transformation tends to differ depending on the model scale; while empirical trends do suggest an easily implemented strategy, a more principled, theory-driven approach is left as future work.

\section{Background}\label{sec:background}

\subsection{Problem formulation}\label{sec:formulation_concise}

Let us start by formulating the ultimate learning problem of interest with high generality. Let $\ZZ$ denote a generic data space. We write $\rdv{Z}$ for a random data point at test time, taking values in $\ZZ$, with distribution $\ddist^{\ast}$. Let $\HH$ denote a set of feasible decisions, and let $\ell: \HH \times \ZZ \to \RR$ be a loss function which assigns a real value $\ell(h;z)$ to (decision, data) pairs. In addition to the data, we have an ancillary \term{random state} variable $\rdv{S}$ associated with $\rdv{Z}$, taking values in a finite set $\Scal$. The ultimate objective we consider is minimization of the \term{(generalized) balanced loss} $\GBL(\cdot)$ on $\HH$, defined by
\begin{align}\label{eqn:GBL_defn}
\GBL(h) \defeq \frac{1}{\abs{\Scal}} \sum_{s \in \Scal} \exx\left[\ell(h;\rdv{Z}) \cond \rdv{S} = s\right].
\end{align}
As an alternative to the traditional expected loss $\exx[\ell(h;\rdv{Z})]$, the objective in (\ref{eqn:GBL_defn}) asks that regardless of the actual distribution of state $\rdv{S}$, data associated with each state $s \in \Scal$ should be given equal weight when evaluating $h \in \HH$. Since this is a learning problem, $\rdv{Z}$ and $\rdv{S}$ are not observable at training time; the learning algorithm is given access to $n$ training points $\rdv{Z}_{1},\ldots,\rdv{Z}_{n}$, assumed to be independent and identically distributed (IID) according to $\ddist$. Traditionally one assumes $\ddist = \ddist^{\ast}$, but allowing $\ddist \neq \ddist^{\ast}$ lets us capture settings with ``distribution drift.'' As with the test data, each training point $\rdv{Z}_{i}$ is associated with a random state $\rdv{S}_{i}$, for $i \in [n]$. When the learning algorithm has access to $n$ independent (data, state) pairs $(\rdv{Z}_{1},\rdv{S}_{1}), \ldots, (\rdv{Z}_{n},\rdv{S}_{n})$, we say it is \term{state-aware}; when the learner only has access to the data $\rdv{Z}_{1},\ldots,\rdv{Z}_{n}$, we say it is \term{state-agnostic}.

Leaving the notion of ``random state'' general allows one to capture many different settings as special cases (discussed in appendix \S{\ref{sec:biblio_notes}}), but to keep our exposition concrete, here we focus on the supervised learning task of classification, and align our formulation with the problem setting considered in the original SharpDRO paper by \citet{huang2023a}. More precisely, the data space $\ZZ = \XX \times \YY$ is composed of an input space $\XX$ and an output space $\YY$, with test data $\rdv{Z} = (\rdv{X},\rdv{Y})$ taking values on $\XX \times \YY$. Each decision $h \in \HH$ is assumed to map an input $x \in \XX$ to a ``label'' $h(x) \in \YY$. The main loss function of interest for evaluation is the zero-one loss, denoted for each $(x,y) \in \XX \times \YY$ and $h \in \HH$ by
\begin{align}
\ell_{\textup{01}}(h;x,y) \defeq 
\begin{cases}
0, & \text{if } h(x) = y\\
1, & \text{else}.
\end{cases}
\end{align}
Plugging $\ell = \ell_{\textup{01}}$ into our balanced loss definition (\ref{eqn:GBL_defn}), we obtain
\begin{align}\label{eqn:GBL_zeroone}
\GBL(h) = \frac{1}{\abs{\Scal}} \sum_{s \in \Scal} \prr{\left\{ h(\rdv{X}) \neq \rdv{Y} \cond \rdv{S} = s\right\}},
\end{align}
and we refer to $1-\GBL(h)$ with $\ell = \ell_{\textup{01}}$ as the \term{balanced accuracy} with respect to the state $\rdv{S}$. In order to align our setup with that of \citet{huang2023a}, here we assume that $\rdv{S}$ represents a non-negative integer \term{corruption severity level}, where $\rdv{X}^{\prime}$ is a ``clean'' example, and the actual data takes the form $\rdv{X} = f(\rdv{X}^{\prime}; \rdv{S})$, where $f(\cdot;s)$ denotes a corruption transformation at level $s \in \Scal = \{0,1,2,\ldots\}$, and $f(x^{\prime};0) = x^{\prime}$ for all $x^{\prime} \in \XX$. See Figure \ref{fig:corrupted_image_examples} for some examples.

\begin{figure}[t!]
\centering
\includegraphics[width=0.33\textwidth]{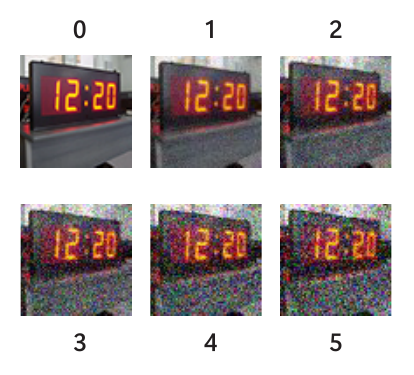}\,\includegraphics[width=0.33\textwidth]{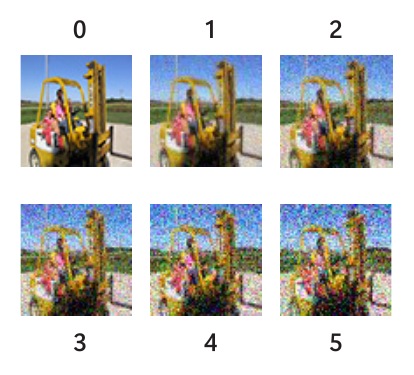}\,\includegraphics[width=0.33\textwidth]{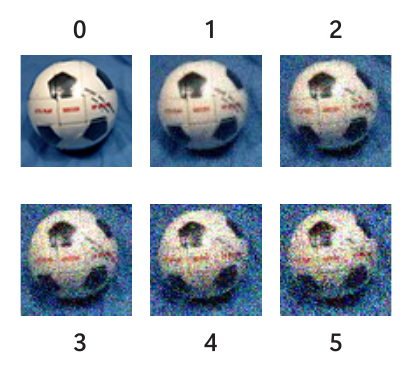}
\caption{Examples of image data at six different corruption severity levels under a Gaussian noise corruption mechanism; note that $\rdv{S} = 0$ amounts to a clean image.}
\label{fig:corrupted_image_examples}
\end{figure}

\subsection{SharpDRO}\label{sec:sdro}
The core \emph{conceptual} proposal of \citet{huang2023a} is a state-aware procedure which attempts to directly minimize the worst-case sharpness
\begin{align}\label{eqn:sdro_conceptual}
\SDRO(h) \defeq \sup_{\Abs{h^{\prime}} \leq \varepsilon} \exx\left[ \ell(h+h^{\prime};\rdv{X},\rdv{Y}) \cond \rdv{S} = s^{\ast}\right]
\end{align}
as a function of $h \in \HH$, where the ``worst state'' $s^{\ast}$ is defined by the property
\begin{align}\label{eqn:sdro_worst_case}
s^{\ast} \in \argmax_{s \in \Scal} \exx\left[ \ell(h;\rdv{X},\rdv{Y}) \cond \rdv{S} = s\right].
\end{align}
Note that $\ell$ here refers to a training loss function, typically a smooth convex surrogate to the zero-one loss $\ell_{\textup{01}}$ mentioned in \S{\ref{sec:formulation_concise}}. We call this their ``conceptual'' proposal because it is the basic motivation underlying the actual procedure that they run.\footnote{For justification of our calling (\ref{eqn:sdro_conceptual})--(\ref{eqn:sdro_worst_case}) the conceptual proposal of \citet{huang2023a}, see \S{\ref{sec:sdro_more_details}}.} In practice, the procedure that they actually implement diverges from this original objective. Here we give a concise overview of what they have actually implemented, based on public source code.\footnote{We are referring to their GitHub repository: \url{https://github.com/zhuohuangai/SharpDRO/}. See \S{\ref{sec:sdro_implementation}} in the appendix for an exposition that cites more specific references in their code.} Assuming we have $n$ training data points $(\rdv{X}_{1},\rdv{Y}_{1}),\ldots,(\rdv{X}_{n},\rdv{Y}_{n})$ and using the full batch for simplicity, given candidate $h_{t}$ at time step $t$, the update to obtain $h_{t+1}$ is as follows:
\begin{align}\label{eqn:sdro_update}
h_{t+1} \defeq h_{t} - \alpha_{t}\left( \widehat{\rdv{G}}_{t} + \widetilde{\rdv{G}}_{t} \right)
\end{align}
where $\alpha_{t}$ is a non-negative step size, the first summand in the update direction vector, namely 
\begin{align}\label{eqn:sdro_first_grad}
\widehat{\rdv{G}}_{t} \defeq \frac{1}{n}\sum_{i=1}^{n}\nabla\ell(h_{t};\rdv{X}_{i},\rdv{Y}_{i})
\end{align}
is simply the gradient of the empirical expected loss on the data batch, evaluated at $h_{t}$. The second term $\widetilde{\rdv{G}}_{t}$ is a bit more involved. Most of the procedure they implement to compute $\widetilde{\rdv{G}}_{t}$ is based on sharpness-aware minimization (SAM) of \citet{foret2021a}, with a few key state-aware modifications, which we now describe. For each $s \in \Scal$, let $\II_{s} \subset [n]$ denote the index of training data in state $s$, i.e., $\II_{s} \defeq \{ i \in [n] \colonset \rdv{S}_{i} = s \}$. With this index notation in place, write the per-state average losses as
\begin{align}\label{eqn:sdro_aveloss_per_state}
\widehat{\rdv{L}}_{s}(h) \defeq \frac{1}{\abs{\II_{s}}} \sum_{i \in \II_{s}} \ell(h;\rdv{X}_{i},\rdv{Y}_{i})
\end{align}
for each $s \in \Scal$. With probabilities initialized at $p_{0,s} \defeq 1 / \abs{\Scal}$ for each $s \in \Scal$, the average losses in (\ref{eqn:sdro_aveloss_per_state}) at step $t$ are used to update a vector of probabilities as
\begin{align}\label{eqn:sdro_adv_probs}
p_{t,s} \defeq \frac{p_{t-1,s}\exp(\beta_{t}\widehat{\rdv{L}}_{s}(h_{t}))}{\sum_{s^{\prime} \in \Scal}p_{t-1,s^{\prime}}\exp(\beta_{t}\widehat{\rdv{L}}_{s^{\prime}}(h_{t}))}
\end{align}
with $\beta_{t} \geq 0$ being another free parameter along with $\alpha_{t}$. With these probabilities now \emph{fixed}, the vector $\widetilde{\rdv{G}}_{t}$ of interest is computed as
\begin{align}\label{eqn:sdro_second_grad}
\widetilde{\rdv{G}}_{t} \defeq \sum_{s \in \Scal} \frac{p_{t,s}}{\abs{\II_{s}}} \sum_{i \in \II_{s}} \nabla\ell(h_{t}+\varepsilon\widehat{\rdv{G}}_{t}/\Abs{\widehat{\rdv{G}}_{t}};\rdv{X}_{i},\rdv{Y}_{i})
\end{align}
with the perturbation vector based on $\widehat{\rdv{G}}_{t}$ given in (\ref{eqn:sdro_first_grad}) and the radius $\varepsilon \geq 0$ given in the original objective (\ref{eqn:sdro_conceptual}).\footnote{Regarding what is being differentiated by autograd, in the SharpDRO GitHub repository file \texttt{train.py}, it is critical to note that the authors pass \texttt{first\_loss.data} to the \texttt{compute\_sharp\_loss()} function, plus use \texttt{adjusted\_loss.data} when computing $p_{t,s}$. As a result, gradients are not computed for these probabilities, which can be seen as a function of neural network weight parameters.} Taking (\ref{eqn:sdro_update}) together with (\ref{eqn:sdro_first_grad}) and (\ref{eqn:sdro_second_grad}), we have described the SharpDRO procedure implemented in public code; the state-agnostic alternative has not been made public. Empirical performance of this practical implementation is competitive, though there is clearly some divergence from their conceptual proposal (\ref{eqn:sdro_conceptual}) mentioned earlier.

\subsection{Difficulties in the state-agnostic setting}\label{sec:issues}

When we are lucky enough to be state-aware, i.e., able to observe $\rdv{S}_{1},\ldots,\rdv{S}_{n}$, it is easy to compute per-state average losses and accuracy to design a training objective and validation procedure aligned with the balanced objectives given in (\ref{eqn:GBL_defn}) and (\ref{eqn:GBL_zeroone}) earlier. The second gradient $\widetilde{\rdv{G}}_{t}$ used by SharpDRO, given in (\ref{eqn:sdro_second_grad}), is a clear example of this, with the per-state probabilities $\{ p_{t,s} \colonset s \in \Scal \}$ designed to put more emphasis on ``the worst-performing states.'' However, when states cannot be observed, the problem of trying to achieve a small balanced error becomes much more difficult. One can consider trying to find a state-agnostic upper bound $\overbar{\GBL}(h)$ such that 
\begin{align}\label{eqn:state_agnostic_bound}
	\GBL(h) \leq \max_{s \in \Scal} \exx\left[\ell(h;\rdv{Z}) \cond \rdv{S} = s\right] \leq \overbar{\GBL}(h)
\end{align}
holds for $h \in \HH$ (e.g., as in \citet{hashimoto2018a}), but the efficacy of this strategy depends strongly on the distribution of the conditional expected loss $\exx\left[\ell(h;\rdv{Z}) \cond \rdv{S} \right]$ over the draw of $\rdv{S}$. If the ``upper tail'' over the random draw of $\rdv{S}$ is too heavy, the ``max-balance gap''
\begin{align}\label{eqn:max_balance_gap}
\max_{s \in \Scal} \exx\left[\ell(h;\rdv{Z}) \cond \rdv{S} = s\right] - \GBL(h) \geq 0
\end{align}
will grow large, causing the objective $\overbar{\GBL}(\cdot)$ to diverge from $\GBL(\cdot)$, and it is well-documented how ignoring too many good-performing examples can be disastrous for off-sample generalization \citep{zhai2021a}. It should also be noted that this max-balance gap can be problematic not just for state-agnostic approaches, but also state-aware methods which target the worst-case expected loss, such as GroupDRO \citep{sagawa2020a} and SharpDRO. With these points in mind, our approach to be described in \S{\ref{sec:coce}} considers how we can ensure the distribution of $\exx\left[\ell(h;\rdv{Z}) \cond \rdv{S}\right]$ (over the random draw of $\rdv{S}$) is not too widely dispersed, without needing to observe $\rdv{S}$.

\section{Concentrated OCE}\label{sec:coce}

Recalling that our ultimate goal as described in \S{\ref{sec:formulation_concise}} is to achieve a small balanced loss, note that when the state-aware expected loss $\exx\left[\ell(h;\rdv{Z}) \cond \rdv{S}\right]$ has a large dispersion, e.g., a large variance over the random draw of $\rdv{S}$, there must be some states $s$ such that the state-conditional expected loss $\exx\left[\ell(h;\rdv{Z}) \cond \rdv{S}=s\right]$ is large, preventing us from achieving a small balanced loss. Since we are assuming a state-agnostic setup, we cannot directly target the distribution of the state-aware expected loss $\exx\left[\ell(h;\rdv{Z}) \cond \rdv{S}\right]$, but instead can only work with the state-agnostic loss $\ell(h;\rdv{Z})$. Note that if we design an objective function which encourages the distribution of $\ell(h;\rdv{Z})$ to be \emph{well-concentrated} near a fixed threshold, this is sufficient, albeit not necessary, to imply the concentration of $\exx\left[\ell(h;\rdv{Z}) \cond \rdv{S}\right]$.\footnote{To see that it is not a necessary condition using an extreme example, losses could be severely heavy-tailed, but if $\ell(h;\rdv{Z})$ is independent of $\rdv{S}$, the balanced loss collapses into the expected loss, i.e., $\GBL(h) = \exx[\ell(h;\rdv{Z})]$, and $\exx\left[\ell(h;\rdv{Z}) \cond \rdv{S}\right]$ has zero variance over the draw of $\rdv{S}$. Recall also that the marginal variance decomposes as $\vaa\left[\ell(h;\rdv{Z})\right] = \exx\left[ \vaa\left[\ell(h;\rdv{Z}) \cond \rdv{S}\right] \right] + \vaa\left[\exx\left[ \ell(h;\rdv{Z}) \cond \rdv{S} \right]\right]$, telling us that if the left-hand side is small, the second term on the right-hand side cannot be large.} On the other hand, when dispersion is measured by average deviations from a threshold, for states which have very few data points, their impact becomes negligible, and some mechanism to emphasize the right-hand tails (worse-performing examples) is necessary. As we describe in \S{\ref{sec:coce_intro}} below, we propose a combination of these two mechanisms.

\subsection{Our approach}\label{sec:coce_intro}

\begin{figure}[t!]
\centering
\includegraphics[width=0.5\textwidth]{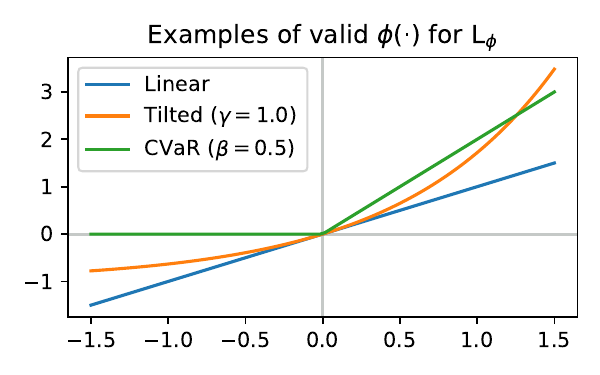}\includegraphics[width=0.5\textwidth]{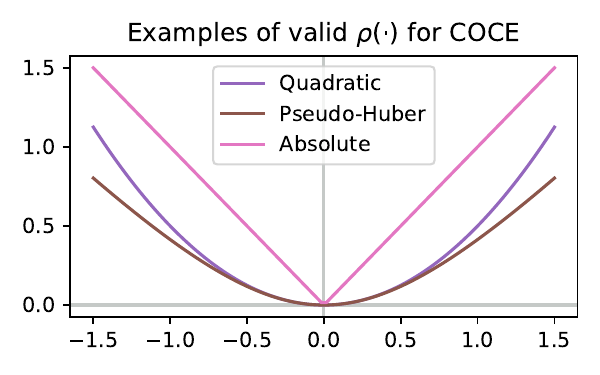}
\caption{Concrete examples of the functions $\phi$ and $\rho$ used in (\ref{eqn:oce_transform}) and (\ref{eqn:concentrated_OCE}) respectively.}
\label{fig:phi_rho_examples}
\end{figure}

As described in the preceding paragraph, our objective function design involves two key steps.
\begin{enumerate}
\item Transform individual losses in a risk-aversive fashion which emphasizes rare but poorly-performing data points.
\item Penalize models under which the transformed losses are not well-concentrated around a sufficiently small threshold.
\end{enumerate}
As a general-purpose strategy, instead of using $\exx[\ell(h;\rdv{Z})]$ (as in ERM) or $\SDRO(h)$ (as in (\ref{eqn:sdro_conceptual}) for SharpDRO), we propose using the following objective function:
\begin{align}\label{eqn:concentrated_OCE}
\COCE(h) \defeq \exx\left[\rho\left(\rdv{L}_{\phi}(h) - \eta\right)\right].
\end{align}
Let us unpack this formulation piece by piece. $\COCE$ is an abbreviation of \term{concentrated OCE}, and OCE stands for \term{optimized certainty equivalent}. Expectation $\exx[\cdot]$ is taken with respect to random data $\rdv{Z}$. We use $\rdv{L}_{\phi}(h)$ to denote the transformed random loss incurred by $h$, and $\rho$ is used to measure deviations around the fixed threshold $\eta$. We describe these two mechanisms in further detail below.

\paragraph{Loss transformation}
Losses are transformed in the style of OCE risk functions.\footnote{See \citet{bental2007a} for an accessible introduction.} Letting $\phi: \RR \to \RR$ denote a function which is convex, non-decreasing, and normalized such that $\phi(0)=0$ and $1 \in \partial\phi(0)$, we define the $\phi$-transformation of base loss $\ell(h;\rdv{Z})$ into $\rdv{L}_{\phi}(h)$ as
\begin{align}\label{eqn:oce_transform}
\rdv{L}_{\phi}(h) \defeq \theta^{\ast} + \phi\left(\ell(h;\rdv{Z}) - \theta^{\ast}\right)
\end{align}
where the degree of shift $\theta^{\ast}$ can be fixed before seeing any data, or if we want to follow the OCE risk style strictly, $\theta^{\ast}$ can be determined based on data as
\begin{align}\label{eqn:oce_internal}
\theta^{\ast} \in \argmin_{\theta \in \RR} \left[ \theta + \exx\left[\phi\left(\ell(h;\rdv{Z}) - \theta\right)\right] \right].
\end{align}
Setting $\phi(u) = u$ recovers the base loss, i.e., $\rdv{L}_{\phi}(h) = \ell(h;\rdv{Z})$, but choices such as $\phi(u) = \max\{0,u\}/(1-\beta)$ with $0 < \beta < 1$ (conditional value-at-risk, CVaR) and $\phi(u) = (\mathrm{e}^{\gamma u}-1)/\gamma$ with $\gamma > 0$ ($\gamma$-tilted risk) make it possible to put more weight on large but infrequent losses. The three examples of $\phi$ just mentioned are shown in the left-hand plot of Figure \ref{fig:phi_rho_examples}.

\paragraph{Penalizing poor concentration}
In the previous literature on learning with OCE risks, the final objective function is simply the expected value of the transformed losses, namely $\exx[\rdv{L}_{\phi}(h)]$.\footnote{See for example \citet{lee2020a,curi2020a,holland2021c,li2023a}.} In contrast, we ask that $\rdv{L}_{\phi}(h)$ be well-concentrated around a fixed threshold $\eta \in \RR$, where ``well-concentrated'' means ``small dispersion as quantified by $\rho:\RR \to [0,\infty)$.'' Since $\rho$ is used to measure dispersion, we ask that $\rho(0)=0$ and that $\rho(\abs{\cdot})$ be non-decreasing; note that we do \emph{not} ask that $\rho(\cdot)$ itself be monotonic. In the right-hand plot of Figure \ref{fig:phi_rho_examples}, we give three typical examples of symmetric $\rho$ (namely $\rho(u) = u^{2}/2$, $\sqrt{u^{2}+1}-1$, and $\abs{u}$).\footnote{The question of when and why to use asymmetric or non-convex $\rho$ is an important and interesting question, but it goes beyond the scope of this paper.} When $\rho$ is differentiable, the properties assumed imply that near the origin, $\rho$ has negative slope below zero and positive slope above zero. This results in an ``ascent-descent'' property that arises when we consider gradient-based optimization methods, since via the chain rule, the derivative of $\rho(\cdot)$ works as a switch between gradient descent and \emph{ascent} by changing signs depending on the sign of $\rdv{L}_{\phi}(h) - \eta$. See \S{\ref{sec:sharpness_links}} for more formal details, and \S{\ref{sec:coce_literature}} for more bibliographic background.

\subsection{Practical implementation}\label{sec:coce_practical}

From a conceptual viewpoint, the key features of our proposed method are well-described by the previous two paragraphs culminating in the objective function $\COCE(\cdot)$ given in (\ref{eqn:concentrated_OCE}). In practice, however, we will not have access to the true data distribution, and thus the expectation over $\rdv{Z}$ will be replaced by the empirical mean taken over a sample $\rdv{Z}_{1},\ldots,\rdv{Z}_{n}$. In addition, the question of how to deal with the internal optimization of $\theta^{\ast}$ in (\ref{eqn:oce_internal}) still needs to be addressed. We discuss three natural strategies below, assuming an implementation using a framework with automatic differentiation such as PyTorch.

\paragraph{Strategy 1 (solve for $\theta^{\ast}$ internally)}
If we wish to follow the procedure laid out in \S{\ref{sec:coce_intro}} in a precise fashion, we can solve for $\theta^{\ast}$ during the forward pass. In some special cases of OCE, such as CVaR ($\phi(u) = \max\{0,u\}/(1-\beta)$) and tilted risk ($\phi(u) = (\mathrm{e}^{\gamma u}-1)/\gamma$), $\theta^{\ast}$ has a known solution which can be easily computed based on the data, and whose (sub-)gradient can be obtained on a backward pass made by standard ``autograd'' procedures.\footnote{For $\beta$-level CVaR, setting $\theta^{\ast}$ to be a $\beta$-level quantile (e.g., via \texttt{torch.quantile}) solves (\ref{eqn:oce_internal}). For $\gamma$-tilted risk, the unique solution is $\theta^{\ast} = \log(\exx[\exp(\gamma \ell(h;\rdv{Z}))])/\gamma$.}

\paragraph{Strategy 2 (optimize $h$ and $\theta$ in parallel)}
Diverging slightly from the procedure in \S{\ref{sec:coce_intro}}, a simple alternative approach is to simply optimize with respect to $h$ and $\theta$ in a joint fashion. That is, denoting $\rdv{L}_{\phi}(h;\theta) \defeq \theta + \phi\left(\ell(h;\rdv{Z}) - \theta\right)$, one could optimize $\exx\left[\rho\left(\rdv{L}_{\phi}(h;\theta) - \eta\right)\right]$ as a function of $(h,\theta)$. While this approach does not guarantee the transformations are strictly of the OCE form, it is easy to implement and works for any choice of $\phi$, particularly when we do not have a closed-form solution to (\ref{eqn:oce_internal}), and an iterative approximation to $\theta^{\ast}$ during each forward pass is cost-prohibitive.

\paragraph{Strategy 3 (leave $\theta$ fixed)}
The simplest approach is simply to fix $\theta$ alongside $\phi$ and $\rho$ from the start, and optimize $\exx\left[\rho\left(\rdv{L}_{\phi}(h;\theta) - \eta\right)\right]$ in $h$, recalling our notation for $\rdv{L}_{\phi}(h;\theta)$ given in our description of Strategy 2. While the mean of the transformed losses $\rdv{L}_{\phi}(h;\theta)$ need not align with a known OCE risk, the desired property of placing emphasis on the worst-case examples can be easily preserved; consider for example fixing $\theta = 0$ and using $\phi(u) = (\mathrm{e}^{\gamma u}-1)/\gamma$ with a positive tilt parameter $\gamma$.

\bigskip

In our empirical investigations to be described shortly in \S{\ref{sec:empirical}}, we test out all the aforementioned strategies, though our main findings center around the empirical performance of the simplest strategy 3. In our public code repository, we include representative implementations of all the above strategies for interested users.

\subsection{Links to sharpness-aware learning}\label{sec:sharpness_links}

There is a large literature on learning algorithms designed to seek solutions around which the surface of the empirical objective function is ``flat,'' a well-known heuristic dating back at least to \citet{hochreiter1994a,hochreiter1997a}. Research in this direction remains active, with first-order methods such as sharpness-aware minimization \citep{foret2021a} showing particularly strong performance for the training of large neural network models. Here we show how our proposed COCE method described in \S{\ref{sec:coce_intro}} can work such that a first-order sharpness penalty is engaged.

Consider implementing the proposed procedure using stochastic gradient descent over the training set. That is, at each step $t$, a random point $\rdv{Z}_{t}$ is sampled from the same distribution as the test point $\rdv{Z}$, and the following update is run:
\begin{align}\label{eqn:sgd_update}
h_{t+1} = h_{t} - \alpha_{t} \rdv{G}_{t}(h_{t}).
\end{align}
Here $\alpha_{t}$ is a non-negative step size, and $\rdv{G}_{t}(h)$ denotes a stochastic (sub-)gradient whose expected value falls into the sub-differential of $\COCE(h)$ in (\ref{eqn:concentrated_OCE}); to keep things concrete, we will write $\rdv{G}_{t}(h) \defeq \rho^{\prime}(\rdv{L}_{\phi,t}(h)-\eta)\nabla\rdv{L}_{\phi,t}(h)$ for each $h$. Unpacking this notation a bit further, $\rho^{\prime}$ is the first derivative of $\rho$, $\eta \in \RR$ is the concentration threshold used in (\ref{eqn:concentrated_OCE}), $\rdv{L}_{\phi,t}(\cdot)$ denotes $\rdv{L}_{\phi}(\cdot)$ in (\ref{eqn:oce_transform}) with $\rdv{Z}$ replaced by $\rdv{Z}_{t}$, and finally $\nabla\rdv{L}_{\phi,t}(h)$ is the stochastic (sub-)gradient of $\rdv{L}_{\phi,t}(\cdot)$ evaluated at $h$.\footnote{In the update (\ref{eqn:sgd_update}) we gloss over technicalities such as sub-differentiability for simplicity of exposition. We are also clearly assuming that the elements $h_{1},h_{2},\ldots$ live in the same space as the sampled stochastic (sub-)gradients; this is all consistent with our exposition of SharpDRO in \S{\ref{sec:sdro}}.} With this notation in place, we give some simple sufficient conditions to induce a sharpness-reducing effect for our procedure (see \S{\ref{sec:additional_proofs}} for a proof).
\begin{prop}[COCE-SGD crossing the threshold]\label{prop:cross_threshold}
With a sequence $(h_{1},h_{2},\ldots)$ generated by the update (\ref{eqn:sgd_update}), assume that for some $t \geq 1$, the transformed loss crosses the $\eta$ threshold in a single step from below; more precisely, assume that the pair $(h_{t},h_{t+1})$ satisfies $\rdv{L}_{\phi,t}(h_{t}) < \eta < \rdv{L}_{\phi,t+t}(h_{t+1})$. Furthermore, set the corresponding pair of step sizes such that $\alpha_{t} = \alpha_{0} / \abs{\rho^{\prime}(\rdv{L}_{\phi,t}(h_{t}) - \eta)}$ and $\alpha_{t+1} = \alpha_{0} / \abs{\rho^{\prime}(\rdv{L}_{\phi,t+1}(h_{t+1}) - \eta)}$, where $\alpha_{0} > 0$ can be set freely. Running one more update to get $h_{t+2}$, we obtain the relation
\begin{align}\label{eqn:cross_threshold}
h_{t+2} = h_{t} - \alpha_{0}^{2}\left[\frac{\nabla\rdv{L}_{\phi,t+1}(h_{t}+\alpha_{0}\nabla\rdv{L}_{\phi,t}(h_{t})) - \nabla\rdv{L}_{\phi,t}(h_{t})}{\alpha_{0}}\right].
\end{align}
\end{prop}

Inspecting the relation (\ref{eqn:cross_threshold}) obtained in Proposition \ref{prop:cross_threshold}, first note that in the numerator of the update direction vector, we have $\nabla\rdv{L}_{\phi,t+1}(h_{t}+\alpha_{0}\nabla\rdv{L}_{\phi,t}(h_{t}))$; this is almost identical to the update direction used in the SAM algorithm of \citet{foret2021a} using mini-batch of size 1, with the only difference being that they use $\nabla\rdv{L}_{\phi,t}(h_{t}+\alpha_{0}\nabla\rdv{L}_{\phi,t}(h_{t}))$, i.e., the leading gradient is $\nabla\rdv{L}_{\phi,t}$, not $\nabla\rdv{L}_{\phi,t+1}$. In the numerator of our (\ref{eqn:cross_threshold}), note that we are taking the difference of $\nabla\rdv{L}_{\phi,t+1}$ (at a perturbed location) and the original $\nabla\rdv{L}_{\phi,t}$. As such, sharpness with respect to both changes in parameter \emph{and} data are being captured here. To make this relation a bit more formal, first define the scaled data deviation $\rdv{D}_{t} \defeq (\rdv{Z}_{t+1} - \rdv{Z}_{t})/\alpha_{0}$, such that we trivially have $\rdv{Z}_{t+1} = \rdv{Z}_{t} + \alpha_{0}\rdv{D}_{t}$. Noting that $\alpha_{0}$ appears in the numerator and denominator, this update direction can be seen as a (first-order) finite-difference approximation of the \emph{gradient of the squared gradient norm}, i.e., the gradient of $\Abs{\nabla\rdv{L}_{\phi,t}(h_{t})}^{2}$, taken with respect to both parameters and data. That is, $h_{t}$ is shifted by $\alpha_{0}\nabla\rdv{L}_{\phi,t}(h_{t})$ and $\rdv{Z}_{t}$ is shifted by $\alpha_{0}\rdv{D}_{t}$, though only the former shift has the SAM-like maximal property (within radius $\alpha_{0}$). These links to sharpness can be considered extensions beyond known connections between the ``Flooding'' algorithm and SAM found by \citet[\S{5.1}]{karakida2023a}, noting that the nature of our COCE-SGD algorithm (qualitatively distinct from Flooding) used in Proposition \ref{prop:cross_threshold} is what lets us capture sharpness with respect to both parameters \emph{and} data.

\section{Empirical tests}\label{sec:empirical}

As we described in \S{\ref{sec:intro}}, our basic motivation for this work is to obtain a more practical alternative to SharpDRO that is state-agnostic, does not use the double-gradient sub-routine of SAM, and performs well under smaller models with less representative power than those studied in the original paper of \citet{huang2023a}. We thus start by re-creating the original state-aware SharpDRO experiments, run the same tests using our state-agnostic COCE method, swapping the large wide ResNet they use with a small CNN model in order to diversify the relative difficulty of ``noisy'' data points. In addition, we go beyond the context of SharpDRO (where $\ddist = \ddist^{\ast}$), to consider a state-agnostic learning task where drift occurs, i.e., where $\ddist \neq \ddist^{\ast}$, and it is only the test data that is corrupted; this is the original use case intended for the robustness benchmarks of \citet{hendrycks2019a}, and on these benchmarks, since there are no severity levels for SharpDRO to use, we simply compare COCE with SAM \citep{foret2021a} instead. In all tests, we also include the special case of COCE with $\phi(u) = u$, referred to as ``soft ascent-descent'' (SoftAD). Experimental setup is described in \S{\ref{sec:empirical_setup}}, and evidence for our key findings is presented in \S{\ref{sec:empirical_findings}} with discussion.

\subsection{Setup}\label{sec:empirical_setup}

\paragraph{Software and hardware}
All our experiments use PyTorch, and are run on three different machines each using a single-GPU implementation.\footnote{\url{https://pytorch.org/}} No parallelization is done. Two machines have an NVIDIA A100 (80GB), and the other machine uses an NVIDIA RTX 6000 Ada. For handling the recording and retrieval of different performance metrics, we use the MLflow library.\footnote{\url{https://mlflow.org/}} All the figures and results presented in this paper can be re-created using source code and Jupyter notebooks that we have made public on GitHub.\footnote{\url{https://github.com/feedbackward/addro}}

\paragraph{Data}
We make use of three well-known clean benchmark datasets: CIFAR-10, CIFAR-100, and ImageNet-30.\footnote{We acquire CIFAR-10/100 via PyTorch (using \texttt{torchvision.datasets.cifar}), and ImageNet-30 via the GitHub repository of \citet{hendrycks2019b} (\url{https://github.com/hendrycks/ss-ood}).} For each dataset, the pre-specified training-test splits are maintained throughout. As for the additional training-validation split, we follow the public code from \citet{huang2023a}, with $79.9\%$ training for CIFAR-10/100 and $76.6\%$ for ImageNet-30. We call these datasets ``clean'' because following the SharpDRO original experimental design, corrupted variants of each of these datasets will be used. Regarding the method of generating corrupted images, the original experiments of \citet{huang2023a} use the noise-generating mechanism of \citet{hendrycks2019a}, with a discrete corruption severity level $\rdv{S}$ sampled independently for each image from a truncated Poisson distribution. More precisely, say $\rdv{S}_{\textup{Pois}}$ is Poisson with shape parameter $1$, i.e., $\prr{\{\rdv{S}_{\textup{Pois}} = s\}} = \mathrm{e}^{-1}/s!$ for each $s \in \{0,1,2,\ldots\}$. Then, since a total of six severity levels (including no corruption) were used in the original tests, the random severity level used can be written as $\rdv{S} = \min\{\rdv{S}_{\textup{Pois}}, 5\}$.\footnote{Under the Poisson setup, $\prr{\{\rdv{S}_{\textup{Pois}} \leq 5\}} \geq 0.999$, so less than one in a thousand samples is truncated.} When $\rdv{S} = 0$, this corresponds to zero corruption, where clean images are used. In all cases, we use the Gaussian noise corruption mechanism.

\paragraph{Models}
We consider two different model classes for image classification. First is a wide ResNet (WRN-28-2) as implemented in the SharpDRO GitHub repository.\footnote{We use the \texttt{WideResNet} class implemented in \texttt{models/resnet.py} as-is. For the original WRN paper, see \citet{zagoruyko2016arxiv}.} Our main interest, however, is in whether or not the robustness properties observed for SharpDRO holds for simpler, less expressive models. To investigate this, we also run tests for a simple convolutional neural network (CNN). Our CNN has two convolutional layers, both with pooling, followed by two hidden fully connected layers. Rectified linear unit (ReLU) activations are used after each of the 2D convolutions and hidden linear layers.

\paragraph{Optimization}
For all tests comparing COCE and SharpDRO, we use optimizer settings which follow the original tests of \citet{huang2023a} precisely. In more detail, we run SGD (\texttt{torch.optim.sgd}) for $200$ epochs with $0.9$ momentum throughout, and an initial learning rate of $0.03$ that is discounted by a multiplicative factor of $0.2$ three times over the course of training, after $30\%$, $60\%$, and $80\%$ progress.\footnote{The learning rate scheduler is included in the SharpDRO GitHub repository (\texttt{step\_lr.py}), but it is not mentioned anywhere in the original paper; our experiments follow the authors' public code.} For all tests, a fixed mini-batch size of $100$ is used. As for our distribution drift experiments comparing COCE with SAM, the above settings are inherited, i.e., the SAM update is wrapped around the SGD optimizer just described. It should also be noted that in the original SharpDRO tests there is a parameter called ``\texttt{step\_size}'' which refers to $\beta_{t}$ from (\ref{eqn:sdro_adv_probs}) in our notation, and this is fixed at $0.01$ for all steps. In all settings, the base loss function used for training is the cross-entropy loss.

\paragraph{Hyperparameters}
Each of the key methods being studied here has a hyperparameter, and for each method we run tests for a range of hyperparameter values. SharpDRO has a radius parameter, given by $\varepsilon$ in (\ref{eqn:sdro_second_grad}), and SAM also has an analogous radius parameter, denoted by ``$\rho$'' in the original paper. Finally, COCE has the key threshold parameter $\eta$ in (\ref{eqn:concentrated_OCE}). Candidate values for hyperparameters are taken from $\{0.01, 0.02, 0.05, 0.1, 0.2, 0.5\}$ for SharpDRO, COCE, and SAM.

\paragraph{Implementation of COCE}
Our proposed method is centered around the objective function $\COCE(h)$ defined in (\ref{eqn:concentrated_OCE}). For the $\phi$-modified losses, we do a simple exponential transformation, with $\rdv{L}_{\phi}(h) = \exp(\gamma\ell(h;\rdv{Z}))$, with $\gamma = 0.1$ fixed throughout. This amounts to a concrete example of ``strategy 3'' discussed earlier in \S{\ref{sec:coce_practical}}, with $\theta = 0$ fixed. We have not experimented with any other choices of $\gamma$ and $\theta$. Note that this exponential transformation is similar to that used in ``tilted ERM'' \citep{lee2020a,li2021a}, a classic case of OCE risk, but the critical difference here is the method of aggregation; instead of re-scaling with $1/\gamma$, averaging, and taking the log, we pass the raw $\phi$-modified losses directly to our dispersion quantifier $\rho$ to penalize concentration. We fix $\rho$ as the smooth pseudo-Huber function $\rho(u) = \sqrt{u^{2}+1}-1$ throughout all tests. There are countless other functions which have similar key properties, such as being approximately quadratic near zero and linear in the limit; our choice is motivated by simplicity. In all cases, expectation $\exx[\cdot]$ in our definition of COCE is with respect to the empirical training data.

\paragraph{Experiment categories}
There are two main categories of experiments that we run. The first is aligned with the original SharpDRO tests of \citet{huang2023a}, where there is no distribution drift, though we run 10 independent trials instead of three. We run vanilla ERM using the optimizer described above, SharpDRO, and COCE for this first category. Only SharpDRO gets access to samples of $\rdv{S}$ at training time. For the second category, we also run 10 independent trials, but now $\ddist \neq \ddist^{\ast}$, with training data being clean, and only test data being corrupted. In this setting, state-aware SharpDRO is replaced with state-agnostic SAM, and only Poisson-corrupted test data is used. Distinct seeds are assigned to each of the random trials for both experiment categories, and so across trials, we get different initial model parameters and different training-validation splits due to random shuffling. Seeds are common across methods (ERM, COCE, SAM, SharpDRO, SoftAD), so every method gets to see the same data for any given trial.

\subsection{Main findings}\label{sec:empirical_findings}

Here we describe our main findings, with figures that highlight the main trends observed in our experiments.

\paragraph{COCE enjoys superior balanced accuracy under small model}
As stated in \S{\ref{sec:background}}, our ultimate evaluation metric of interest is the balanced accuracy, namely $1-\GBL(h)$ under the 0-1 loss $\ell = \ell_{\textup{01}}$, given in (\ref{eqn:GBL_zeroone}). In Figure \ref{fig:balacchyp_cnn_im30}, we show how the test balanced accuracy (larger is better) compares across different hyperparameter settings at different stages of training (after 50th epoch, 150th epoch, and the last epoch). Balanced accuracy is displayed using error bars, representing the mean $\pm$ standard deviation (over trials). Note that the horizontal black line denotes the balanced accuracy achieved by ERM, which has no hyperparameters. As hypothesized, naively putting too much weight on the worst case under a simple CNN is clearly detrimental to the performance of SharpDRO, whereas for the settings used here, our COCE shows strong, robust performance. This is significant because while SharpDRO does not attempt to optimize balanced accuracy directly, it is still an ``oracle'' in the sense that it is given direct access to the severity level $\rdv{S}$ associated with each observation $(\rdv{X},\rdv{Y})$, and uses observations of $\rdv{S}$ to do the per-state re-weighting used in (\ref{eqn:sdro_second_grad}) from \S{\ref{sec:sdro}}. In contrast, all the other methods only see $(\rdv{X},\rdv{Y})$. This trend holds across all datasets tested.

\begin{figure}[t!]
\centering
\includegraphics[width=1.0\textwidth]{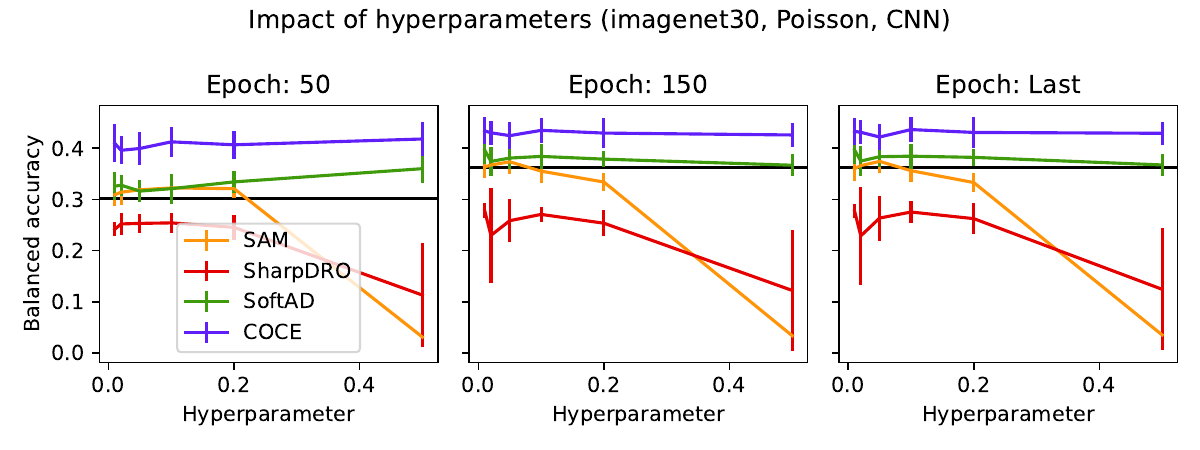}
\caption{Balanced accuracy for different hyperparameters under a simple CNN model.}
\label{fig:balacchyp_cnn_im30}
\end{figure}

\paragraph{COCE displays unique performance under auxiliary metrics}
While balanced accuracy is the main metric of interest here, other evaluation metrics are naturally of interest, in particular the performance \emph{on average}, instead of state-balanced performance. In Figure \ref{fig:mainexp_im30_compare}, we show how each method performs in three auxiliary metrics as training progresses. The top row shows the average (cross-entropy) loss. The middle row shows the usual classification accuracy, in contrast to the state-balanced accuracy in Figure \ref{fig:balacchyp_cnn_im30}. The bottom row shows the usual L2 norm of model parameters. Model selection is done per-trial, based on the highest average accuracy on validation data; the plotted values are averaged over trials. Note that under the CNN model, COCE shows superior performance \emph{on average}, as well, indicating an appealing alternative to ERM and SAM. Furthermore, for both the CNN and WRN models, COCE shows a much stronger model regularization effect, leading to a much smaller norm; these trends hold across all the datasets used.

\begin{figure}[t!]
\centering
\includegraphics[width=0.5\textwidth]{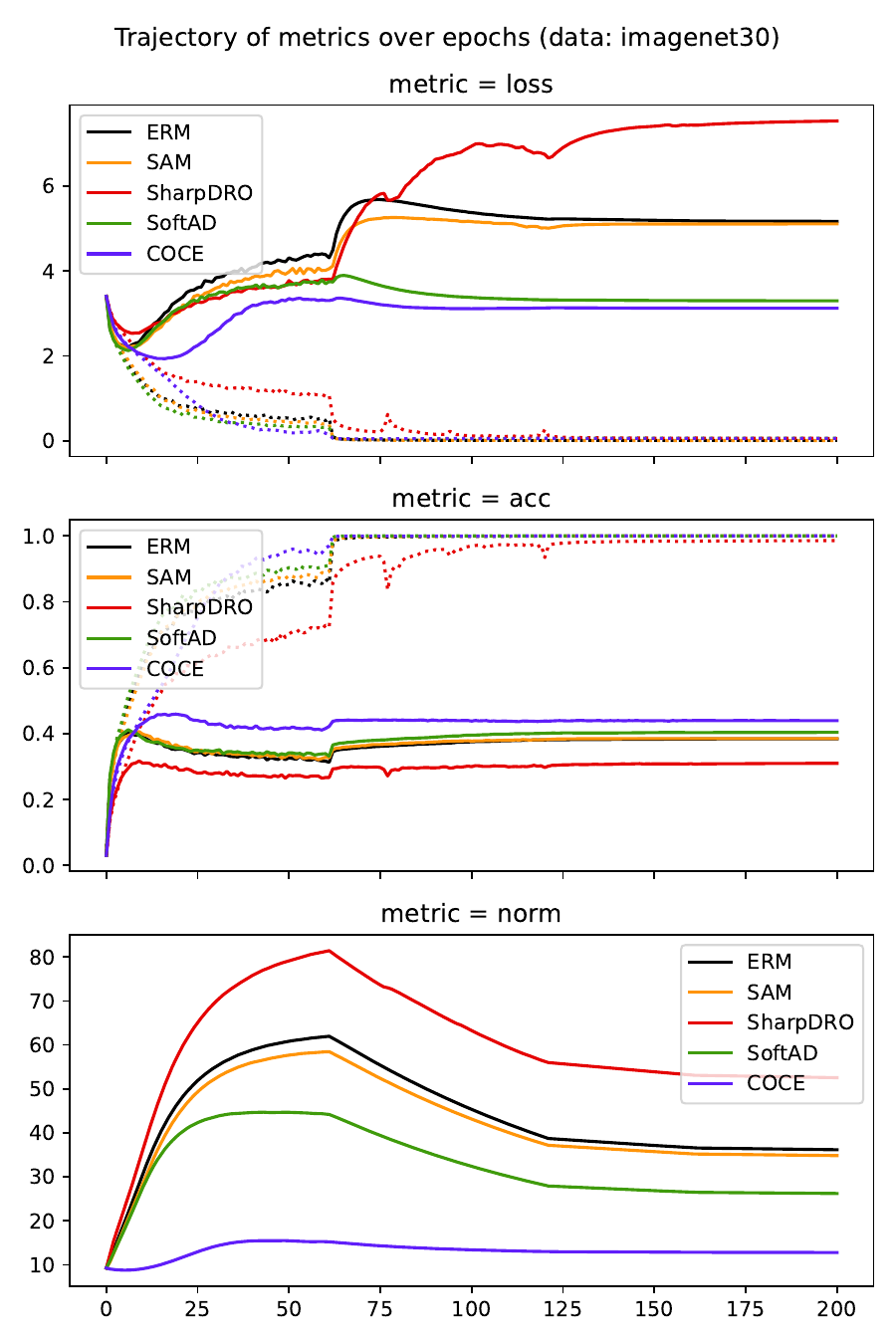}\includegraphics[width=0.5\textwidth]{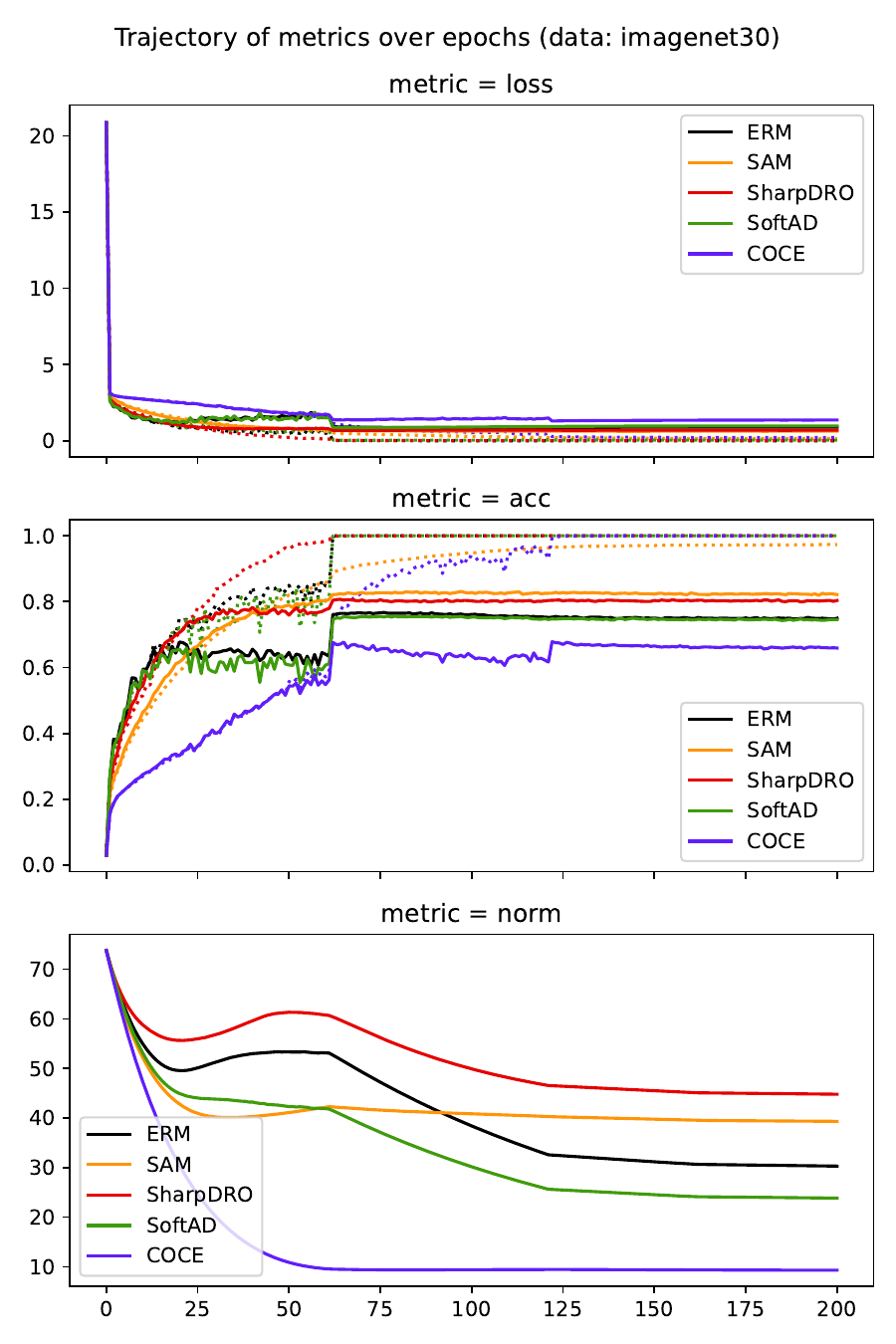}
\caption{Average loss, accuracy, and model norm trajectories (left: CNN, right: WRN). Dashed lines are based on training data, solid lines are based on test data.}
\label{fig:mainexp_im30_compare}
\end{figure}

\paragraph{Limitations under overparameterized models}
Comparing the left side (CNN model) and right side (WRN model) of Figure \ref{fig:mainexp_im30_compare}, we see how the positions of SharpDRO and COCE flip in terms of accuracy on average. Strong performance for SharpDRO in the WRN case was already established by \citet{huang2023a}, and we have naively inherited their hyperparameter range and optimizer settings for COCE, so the strength of SharpDRO here is not surprising. With a broader range of COCE settings ($\eta$ and $\theta$), it is possible to achieve much better performance, but establishing which settings to use under which model goes beyond the scope of this paper.

\paragraph{Robustness under distribution shift}
Finally, we diverge from the SharpDRO experimental setup completely, and instead return to the original intended use of the corrupted ImageNet and CIFAR-10/100 datasets \citep{hendrycks2019a}, which is evaluation, not training. We use the original clean datasets for training/validation, and use the corrupted datasets for testing only. As mentioned earlier, since the state $\rdv{S}$ is irrelevant for the training data, state-aware SharpDRO is not applicable, so instead we compare COCE with SAM. In Figure \ref{fig:balacchyp_cnn_cifar10_drift}, we show results analogous to those in Figure \ref{fig:balacchyp_cnn_im30}, except now for the CIFAR-10 dataset under distribution drift; the same trends hold for other datasets as well. At the very least, we see that under a simple CNN model, COCE provides a very appealing low-cost alternative to SAM that does well under distribution shift.

\begin{figure}[t!]
\centering
\includegraphics[width=1.0\textwidth]{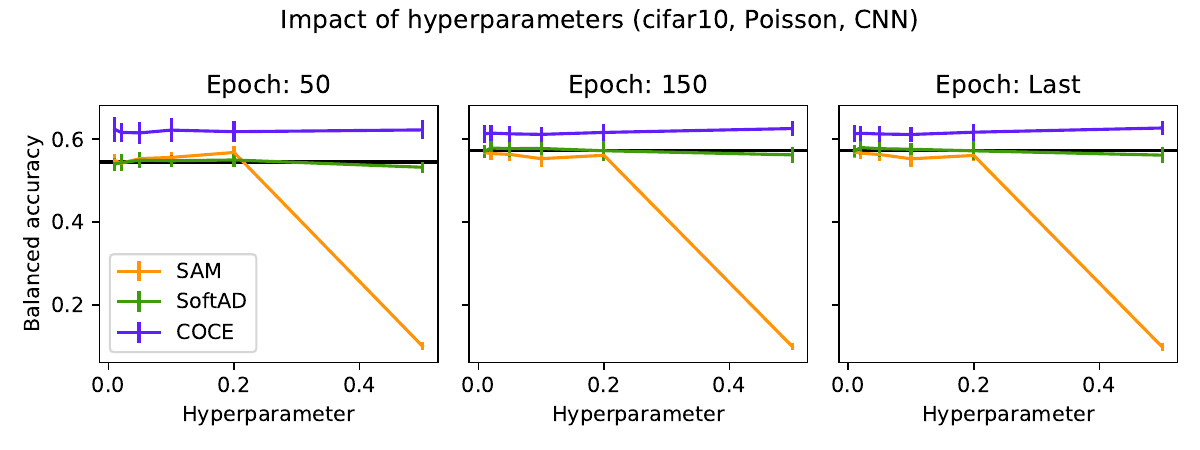}
\caption{Performance of state-agnostic methods on CIFAR-10 under distribution drift.}
\label{fig:balacchyp_cnn_cifar10_drift}
\end{figure}

\section{Conclusion}\label{sec:conclusion}

In this work, we investigated how the state-aware SharpDRO procedure of \citet{huang2023a} performs outside the original scenario of ``over-parameterized model, difficult points are present in both training/testing.'' We saw how using a model with less expressive power leads to far less desirable test performance in terms of balanced accuracy, and introduced concentrated OCE (COCE) as a flexible, state-agnostic alternative, which was empirically shown to be a strong practical alternative across datasets. We established formal connections between our procedure and first-order sharpness penalties, and showed empirically how COCE has the potential to be a low-cost alternative to SAM in the context of distribution drift under small models.

Taking these points together, we have initial evidence supporting COCE as a learning algorithm with robust generalization properties. That said, our investigation is still in a nascent stage. From the empirical side, the biggest limitation is that we have naively inherited the settings of the SharpDRO experiments, and have only studied two very simple special cases of COCE (with and without exponential tilting). While initial results are encouraging, the slope of $\phi$ is expected to have a strong impact on performance, and strategies for deciding on $\phi$ and $\eta$ in a joint fashion are an important future direction. As we discussed in \S{\ref{sec:background}}, our formulation of the ``generalized balanced loss'' is very broad, and captures many important problem settings related to class imbalance and fairness. At the same time, our COCE method is not in any way specialized to the image classification problem, and thus ``concentrated'' variants of existing algorithms seeking good balanced loss performance is another appealing direction we intend to pursue.

\section*{Acknowledgments}
This work was supported by JSPS KAKENHI (grant numbers 22H03646 and 23K24902), JST PRESTO (grant number JPMJPR21C6), and a grant from the SECOM Science and Technology Foundation.

\clearpage

\appendix

\section{Bibliographic notes}\label{sec:biblio_notes}

\subsection{Previous work related to the generalized balanced loss}

While we have already described in \S{\ref{sec:formulation_concise}} the particular focus of this paper, set within the context of closely related literature, the generalized balanced loss (\ref{eqn:GBL_defn}) given earlier is flexible enough to capture many important learning problems as special cases. Here we briefly review key categories of such learning problems, while highlighting some representative existing work.

\paragraph{Fairness in machine learning}
Our generalized balanced loss (\ref{eqn:GBL_defn}) can be seen as the expected value of the conditional expected loss $\exx[\ell(h;\rdv{Z}) \cond \rdv{S}]$ (randomness due to $\rdv{S}$) under the assumption that $\rdv{S}$ is uniformly distributed over finite set $\Scal$. This random variable, without the uniform assumption, also plays a key role in the literature on fairness. Many typical definitions of ``fair'' decisions ask that performance should not vary too much across ``sensitive sub-groups.'' For example, random state $\rdv{S}$ can be used to model association with such a sub-group, with perfect fairness amounting to $\exx[\ell(h;\rdv{Z}) \cond \rdv{S} = s] = \exx[\ell(h;\rdv{Z}) \cond \rdv{S} = s^{\prime}]$ for all $s, s^{\prime} \in \Scal$. When losses are bounded below, one natural approach is to tackle the expected loss on the worst-performing sub-group, namely to minimize $\max_{s \in \Scal}\exx[\ell(h;\rdv{Z}) \cond \rdv{S} = s]$ as a function of $h \in \HH$. In the state-aware case, the worst-case group can be directly targeted \citep{sagawa2020a}, but in the state-agnostic case, we cannot directly sample $\rdv{S}$ or $\exx[\ell(h;\rdv{Z}) \cond \rdv{S}]$. In such a setting, one can use alternative criteria such as $\chi^{2}$-DRO (distributionally robust optimization) or conditional value-at-risk (CVaR) based on the original random loss $\ell(h;\rdv{Z})$ without conditioning on $\rdv{S}$ \citep{hashimoto2018a,williamson2019a}, with the caveat that parameters of these criteria need to be set such that the degree of tail sensitivity is sufficiently high. In terms of the literature on justice and fairness outside of machine learning, the worst-case or ``maximin'' approach is closer to the thinking of John Rawls, whereas our balanced objective (\ref{eqn:GBL_defn}) takes us closer to the principle of being an ``impartial observer,'' a notion closely associated with John C.~Harsanyi.\footnote{We note that the idea of a balanced objective (over sensitive sub-groups) similar to our (\ref{eqn:GBL_defn}) is mentioned briefly by \citet[\S{6.1}]{williamson2019a}. For more on the theory of justice and fairness, see for example the following entries in the Stanford Encyclopedia of Philosophy: \url{https://plato.stanford.edu/entries/rawls/} and \url{https://plato.stanford.edu/entries/impartiality/}.} While fairness notions are outside the scope of our evaluations in this paper, the state-agnostic method we describe in \S{\ref{sec:coce}} can easily be applied to such settings. We leave such applications for future work.

\paragraph{Dealing with class imbalance}
In the context of classification under class imbalance, our generalized balanced classification error (\ref{eqn:GBL_zeroone}) includes the well-known ``class-balanced error'' \citep{menon2013a} as an important state-aware special case. More precisely, if the data comes in the form of an (input, label) pair $(\rdv{X},\rdv{Y})$ taking values in $\XX \times \YY$, and our random state is simply understood to be the class label, namely $\rdv{S} = \rdv{Y}$ and $\Scal = \YY$, then under the zero-one loss $\ell = \ell_{\textup{01}}$ we have $\GBL(h) = \CBE(h)$, where the latter is defined by
\begin{align}\label{eqn:CBE_defn}
\CBE(h) \defeq \frac{1}{\abs{\YY}} \sum_{y \in \YY} \prr{\left\{ h(\rdv{X}) \neq \rdv{Y} \cond \rdv{Y} = y \right\}}.
\end{align}
Clearly, the preceding $\CBE(h)$ aims at ensuring the error is good across all classes, regardless of whether or not they are majority or minority classes. While designing a learning algorithm that is consistent in terms of $\GBL(\cdot)$ is not in general feasible under a state-agnostic setting, $\CBE(h)$ depends only on the observable labels, and with label information in hand it is possible to modify typical losses such as the softmax cross-entropy to design a (Fisher) consistent procedure \citep{menon2021a}.

\paragraph{Robustness to distribution drift}
While we have seen in the preceding paragraphs that achieving strong performance in terms of balanced loss across states is a natural goal in and of itself, balanced loss minimization can also be seen as one approach to realize better performance in various metrics under distribution drift (i.e., when $\ddist \neq \ddist^{\ast}$). In particular for image classification, there is a massive literature on designing neural network training procedures which are ``robust'' in the sense that they perform well on test examples with traits that are extremely rare or non-existent in the training data \citep{hendrycks2019a}. Naturally, the presence or degree of these aforementioned ``traits'' (e.g., texture changes, blur, additive random noise, etc.) can be modelled as state $\rdv{S}$, and when traits are easily pinned down and re-created using simple functions (essentially a state-aware scenario), practical tricks such as data augmentation are very effective \citep{hendrycks2021a}. In more realistic settings where data variations between $\ddist$ and $\ddist^{\ast}$ are not easily foreseen, the resulting state-agnostic learning task is much more difficult. When out-of-distribution data is available, the unobserved states can potentially be \emph{estimated} and then leveraged to augment training data or regularize the training process, as done by \citet{zhu2023a}. From our viewpoint, all these efforts can be seen as attempts to modify the training data distribution such that the usual empirical risk objective $\sum_{i=1}^{n}\ell(h;\rdv{Z}_{i})/n$ does not diverge too far from the \emph{trait}-balanced empirical variant of $\GBL(h)$. Our state-agnostic approach described in \S{\ref{sec:coce}} can be considered orthogonal to such existing methods, and in \S{\ref{sec:empirical}} we include tests under distribution drift where $\ddist$ and $\ddist^{\ast}$ do not align.

\subsection{Past work related to COCE}\label{sec:coce_literature}

In our introduction to COCE in \S{\ref{sec:coce}}, one central concept is that of penalizing losses with poor concentration. There is a massive literature related to designing objectives based on the mean and variance of the loss distribution, in particular for implying bounds with ``fast rates'' \citep{maurer2009a}, though in practice, variance is almost always approximated from above using worst-case oriented risks such as OCE or DRO risks \citep{gotoh2018a,duchi2019a}, under the key assumption that losses are bounded below. In contrast, by penalizing dispersion in a direct way, as we do in our objective (\ref{eqn:concentrated_OCE}), the support of the loss distribution need not be bounded from above or below. This direct approach is not common in the literature; the earliest example we can find is \citet{holland2019b}, who considers training classifiers to have margins that are well-concentrated. Subsequent work has formalized and generalized this class of learning criteria \citep{holland2022c,holland2023bdd,holland2024bddflood}. That said, such procedures have been positioned in the existing literature as \emph{alternatives} to OCE-type transforms, and to the best of our knowledge our proposal is the first to consider the utility of combining these two techniques.

Another important notion that arises in \S{\ref{sec:coce}} is that of iterating between gradient ascent and descent. While this is admittedly not yet established as a commonplace technique, it does appear in the machine learning literature in various forms. For example, in the context of seeking ``flat minima,'' finite-difference approximations to sharpness-related quantities are frequently used \citep{karakida2023a}, and typically involve both ascent and descent steps by definition. The sharpness-aware minimization procedure of \citet{foret2021a} is one important example, as is the ``Flooding'' mechanism of \citet{ishida2020a}, which using our notation seeks to optimize $\abs{\exx[\ell(h;\rdv{Z})]-\eta}$. Note the key qualitative difference from $\COCE(h)$; Flooding just asks the loss mean to be close to $\eta$, whereas $\COCE(h)$ accounts for poor concentration around the same threshold.

\section{Theory}

\subsection{Proof details}\label{sec:additional_proofs}

\begin{proof}[Proof of Proposition \ref{prop:cross_threshold}]
First note that immediate from just the definition of the SGD update (\ref{eqn:sgd_update}) and the particular form of $\rdv{G}_{t}$ and $\rdv{G}_{t+1}$, we know that
\begin{align}
\nonumber
h_{t+2} & = h_{t+1} - \alpha_{t+1}\rdv{G}_{t+1}(h_{t+1})\\
\label{eqn:cross_threshold_unfold}
& = h_{t} - \alpha_{t}\rdv{G}_{t}(h_{t}) - \alpha_{t+1}\rdv{G}_{t+1}(h_{t} - \alpha_{t}\rdv{G}_{t}(h_{t}))\\
\nonumber
& = h_{t} - \alpha_{t}\rho^{\prime}(\rdv{L}_{\phi,t}(h_{t})-\eta)\nabla\rdv{L}_{\phi,t}(h_{t})\\
\label{eqn:cross_threshold_scaling}
& \hspace{0.9cm} - \alpha_{t+1}\rho^{\prime}(\rdv{L}_{\phi,t+1}(h_{t+1})-\eta)\nabla\rdv{L}_{\phi,t+1}(h_{t} - \alpha_{t}\rho^{\prime}(\rdv{L}_{\phi,t}(h_{t})-\eta)\nabla\rdv{L}_{\phi,t}(h_{t})).
\end{align}
To clean up the right-hand side of (\ref{eqn:cross_threshold_scaling}), recall our assumption about how $\alpha_{t}$ and $\alpha_{t+1}$ are to be set, with $\alpha_{0}$ re-scaled by the absolute value of the $\rho^{\prime}(\cdot)$ factors. This lets us cancel out the $\rho^{\prime}(\cdot)$ terms, replacing them with factors that depend only on the sign of the loss dispersion. More precisely, writing $r_{t} \defeq \sign(\rdv{L}_{\phi,t}(h_{t+1})-\eta)$ and $r_{t+1} \defeq \sign(\rdv{L}_{\phi,t+1}(h_{t+1})-\eta)$, we have
\begin{align}
\label{eqn:cross_threshold_general}
h_{t+2} = h_{t} - \alpha_{0} \left[ r_{t}\nabla\rdv{L}_{\phi,t}(h_{t}) + r_{t+1}\nabla\rdv{L}_{\phi,t+1}(h_{t} - \alpha_{0}r_{t}\nabla\rdv{L}_{\phi,t}(h_{t})) \right].
\end{align}
At this point, note that the assumption made of
\begin{align*}
\rdv{L}_{\phi,t}(h_{t}) < \eta < \rdv{L}_{\phi,t+t}(h_{t+1})
\end{align*}
is equivalent to $r_{t} = -1$ and $r_{t+1} = 1$ holding together. Plugging this in to (\ref{eqn:cross_threshold_general}), we have
\begin{align*}
h_{t+2} = h_{t} - \alpha_{0} \left[ \nabla\rdv{L}_{\phi,t+1}(h_{t} + \alpha_{0}\nabla\rdv{L}_{\phi,t}(h_{t})) - \nabla\rdv{L}_{\phi,t}(h_{t}) \right],
\end{align*}
from which the desired result follows trivially.
\end{proof}

\subsection{Conceptual proposal of SharpDRO}\label{sec:sdro_more_details}

We say (\ref{eqn:sdro_conceptual}) and (\ref{eqn:sdro_worst_case}) characterize the ``conceptual proposal'' of \citet{huang2023a} based on the theoretical exposition in their paper. Their notation is somewhat inconsistent and contains undefined terms (e.g., ``$\mathcal{L}(\theta;x,y)$'' is defined but ``$\mathcal{L}(\theta,\omega;x,y)$'' is not), but citing the following points as evidence, we are confident our interpretation is valid.
\begin{itemize}
\item From their \S{3.1}, they say that ``instead of directly applying sharpness minimization on the whole dataset,'' they choose to ``focus on sharpness minimization over the worst-case distribution.''
\item As an immediate follow-up to the previous point, in the state-aware case (they refer to this as the ``distribution-aware'' case in \S{3.2.1}), their definition of the ``worst-case distribution'' is said to be found by identifying the ``sub-distribution'' (indexed by $s \in \Scal$ in our notation) which ``yields the largest loss.'' The term ``largest loss'' is vague, since losses are defined per-point, but after their equation (10), they say the first term of their (10) ``simply recovers the learning target of GroupDRO,'' meaning we can take the ``worst-case distribution'' to mean the choice of $s \in \Scal$ which yields the largest (conditional) expected loss, just as we have specified in (\ref{eqn:sdro_worst_case}).
\item The objective they write in their equation (10) contains two terms, one with ``$\mathcal{L}$'' (GroupDRO objective) and one with ``$\mathcal{R}$'' (sharpness regularizer), but given the definition of the latter in their equation (9), losses evaluated at the current parameter (their ``$\theta$'') should cancel out, leaving just the conditional worst-case expected loss evaluated at the perturbed position.
\item As a separate point, the objective as stated in their (5) (denoted by ``$\mathcal{L}_{\textup{SharpDRO}}$'') also implies the same conclusion, plugging in their definition (9) of ``$\mathcal{R}$''.
\end{itemize}
Taken together, these points suggest strongly that the conceptual objective function intended by the authors is as we have described in (\ref{eqn:sdro_conceptual}) and (\ref{eqn:sdro_worst_case}). We provide further details related to their actual implementation in \S{\ref{sec:sdro_implementation}} to follow.

\subsection{Implementation of SharpDRO}\label{sec:sdro_implementation}

To justify our exposition in \S{\ref{sec:sdro}} of how SharpDRO is implemented by \citet{huang2023a}, we provide a concise summary of where in their source code the key elements can be confirmed. At the time of writing, the most recent commit to their GitHub repository (\url{https://github.com/zhuohuangai/SharpDRO}) is \texttt{36eefa2} (March 27, 2023), and all our comments are based on this version.
\begin{itemize}
\item To begin, from \texttt{main.py} lines 131--132, it is clear that the training procedure is done by running the \texttt{train} method of an instance of the \texttt{Trainer} class defined in \texttt{train.py}.

\item Digging in to \texttt{train.py}, from line 156, we see that an instance of \texttt{SAM} is used as the main optimizer, but critically this is not the original SAM, but rather a slightly modified version. This modified version is in their \texttt{sam.py}, and the key additions they make are
\begin{verbatim}
  self.state[p]["old_g"] = p.grad.clone()
\end{verbatim}
within the definition of \texttt{first\_step} (line 76), and
\begin{verbatim}
  p.grad.add_(self.state[p]["old_g"])
\end{verbatim}
within the definition of \texttt{second\_step} (line 88). These two lines taken together in their respective contexts amount to saving the current mini-batch gradient before sharpness-aware perturbation ($\widehat{\rdv{G}}_{t}$ in our notation), and then adding it back after the sharpness-aware gradient has been computed; this add-back is why in (\ref{eqn:sdro_update}) the update direction is the sum $\widehat{\rdv{G}}_{t} + \widetilde{\rdv{G}}_{t}$, and not simply $\widetilde{\rdv{G}}_{t}$. This add-back is mentioned by the authors in \S{B.1} of the supplementary materials to their CVPR paper.

\item Regarding the nature of how $\widetilde{\rdv{G}}_{t}$ is computed, we have already discussed the key details of this in \S{\ref{sec:sdro}}.

\item Given the points established above, the precise form  (\ref{eqn:sdro_update}) follows immediately from line 90 of their \texttt{sam.py} where \texttt{self.base\_optimizer.step()} is run. Note that the ``base optimizer'' is \texttt{torch.optim.SGD} (line 155 of \texttt{train.py}), meaning that after multiplying $\widehat{\rdv{G}}_{t} + \widetilde{\rdv{G}}_{t}$ by the specified learning rate, the result is subtracted from the current parameters.
\end{itemize}
The above implementation of state-aware SharpDRO is clearly very strong in practice, but it diverges from the conceptual proposal of SharpDRO in \S{\ref{sec:sdro}} and \S{\ref{sec:sdro_more_details}}.

\bibliography{../refs/refs}

\end{document}